\documentclass{article}

\usepackage{arxiv}

\usepackage[utf8]{inputenc} 
\usepackage[T1]{fontenc}    
\usepackage{hyperref}       
\usepackage{url}            
\usepackage{booktabs}       
\usepackage{amsfonts}       
\usepackage{nicefrac}       
\usepackage{microtype}      
\usepackage{lipsum}

\usepackage{cite}
\usepackage{amsmath,amssymb,amsfonts}
\usepackage{graphicx}
\usepackage{algorithm}
\usepackage[noend]{algpseudocode}
\usepackage{hyperref}
\usepackage{textcomp}
\usepackage{threeparttablex} 

\usepackage{mathtools}
\usepackage{caption}
\usepackage{float}
\usepackage{stmaryrd}
\usepackage{epstopdf}
\usepackage{multirow}
\usepackage{pifont}
\usepackage{caption}
\usepackage{subcaption}
\usepackage{xcolor}

\newcommand{\col}{\textsf{ col}}

\newcommand{\R}{{\mathbb{R}}}

\newcommand{\N}{{\mathbb{N}}}

\newcommand{\tx}{{t_{\times}}}

\newtheorem{theorem}{Theorem}[section]

\newtheorem{assumption}{Assumption}

\newtheorem{remark}[theorem]{Remark}
\newtheorem{problem}[theorem]{Problem}
\newtheorem{proof}[theorem]{Proof}

\title{Learning from Demonstration via Spatiotemporal Tubes for Unknown Euler–Lagrange Systems
\thanks{ This work was supported in part by the SERB Start-Up Research Grant; in part by the ARTPARK. The work of Ratnangshu Das was supported by the Prime Minister’s Research Fellowship from the Ministry of Education, Government of India.}
}

\author{
 Ratnangshu Das \\
  Robert Bosch Centre for Cyber-Physical Systems\\
  IISc, Bengaluru, India\\
  \texttt{ratnangshud@iisc.ac.in} \\
   \And
 Puneeth Shankar \\
  Robert Bosch Centre for Cyber-Physical Systems\\
  IISc, Bengaluru, India\\
  \texttt{puneethz1z3z2@gmail.com} \\
  \And
 Varuni Buereddy \\
  Robert Bosch Centre for Cyber-Physical Systems\\
  IISc, Bengaluru, India\\
  \texttt{varunib@iisc.ac.in} \\
  \And
 Ravi Prakash \\
  Robert Bosch Centre for Cyber-Physical Systems\\
  IISc, Bengaluru, India\\
  \texttt{ravipr@iisc.ac.in} \\
  \And
 Pushpak Jagtap \\
  Robert Bosch Centre for Cyber-Physical Systems\\
  IISc, Bengaluru, India\\
  \texttt{pushpak@iisc.ac.in} \\
}

\begin{document}
\maketitle

\begin{abstract}
 We present STT-LfD, a unified Learning from Demonstration (LfD) framework that integrates motion learning with control for unknown Euler–Lagrange systems. Unlike traditional decoupled approaches that track a fixed reference, the proposed method treats demonstrations as a data-driven safety specification. Using heteroscedastic Gaussian Processes, STT-LfD learns Spatiotemporal Tubes (STTs) as an intent envelope that capture time-varying precision requirements of a task. A closed-form feedback controller then enforces these learned constraints while respecting actuator limits, without requiring explicit system identification. The approach preserves the temporal structure of demonstrations, remains computationally efficient, and avoids explicit system identification. Hardware experiments on a mobile robot and a 7-DOF manipulator show that it outperforms baselines in robustness to disturbances and computational speed. \href{https://tinyurl.com/STT-LfD}{Video Link}perturbed conditions.
\end{abstract}

\section{Introduction}

Robots operating in safety-critical and unstructured environments often need to perform complex skills that are easier to demonstrate than to manually program. Learning from Demonstration (LfD) offers a natural way to transfer human expertise directly into robotic policies \cite{ravichandar2020recent, celemin2022interactive}. However, traditional LfD architectures typically suffer from a rigid decoupling: a high-level motion generation stage (e.g., DMP \cite{ijspeert2002movement, schaal2006dynamic} or KMP) produces a reference trajectory, which a separate low-level controller (e.g., LQR \cite{silverio2019uncertainty} or MPC \cite{hu2019mobile}) then attempts to track. This separation is inherently fragile, as the motion generator ignores the hardware's physical constraints, while the tracker treats all trajectory points with uniform importance. 
As a result, these approaches fail to capture what we call the intent envelope: a time-varying spatial tolerance implicitly expressed by the expert. For example, during precise manipulation (e.g., threading a needle), this envelope is narrow and restrictive, whereas in free-space motion, it can be much wider.

In this letter, we address this limitation by introducing STT-LfD, a unified framework that bridges the gap between probabilistic motion representation and robust control for unknown Euler-Lagrange systems. We use heteroscedastic Gaussian Processes (HGPs) \cite{heteroscedastic} to learn Spatiotemporal Tubes (STTs) \cite{das2025spatiotemporal, das2025real} from demonstrations that capture both the nominal task behavior and its time-varying precision requirements. 
{However, previous STT formulations primarily focused on satisfying formally specified temporal logic constraints through optimization- or learning-based tube synthesis. In contrast, the proposed framework constructs STTs directly from expert demonstrations, allowing the learned tubes to capture both nominal task behavior and demonstration variability.}
The learned tubes naturally distinguish between task-critical regions, where tracking must remain precise, and flexible regions where higher variance is allowed.

Unlike conventional LfD methods that require explicit system identification and real-time optimization, as in LQR \cite{silverio2019uncertainty}, MPC \cite{hu2019mobile}, and CBFs \cite{robey2020learning}, the proposed STT-LfD framework provides a closed-form execution strategy that is mathematically coupled with the learned tube and does not require parametric model identification. This unified design allows the robot to satisfy the spatial and temporal requirements of the demonstrated task, while respecting actuator limits and remaining robust to disturbances, even when the system parameters are unknown.
\begin{figure*}[t]
    \centering
    \includegraphics[width=0.9\linewidth]{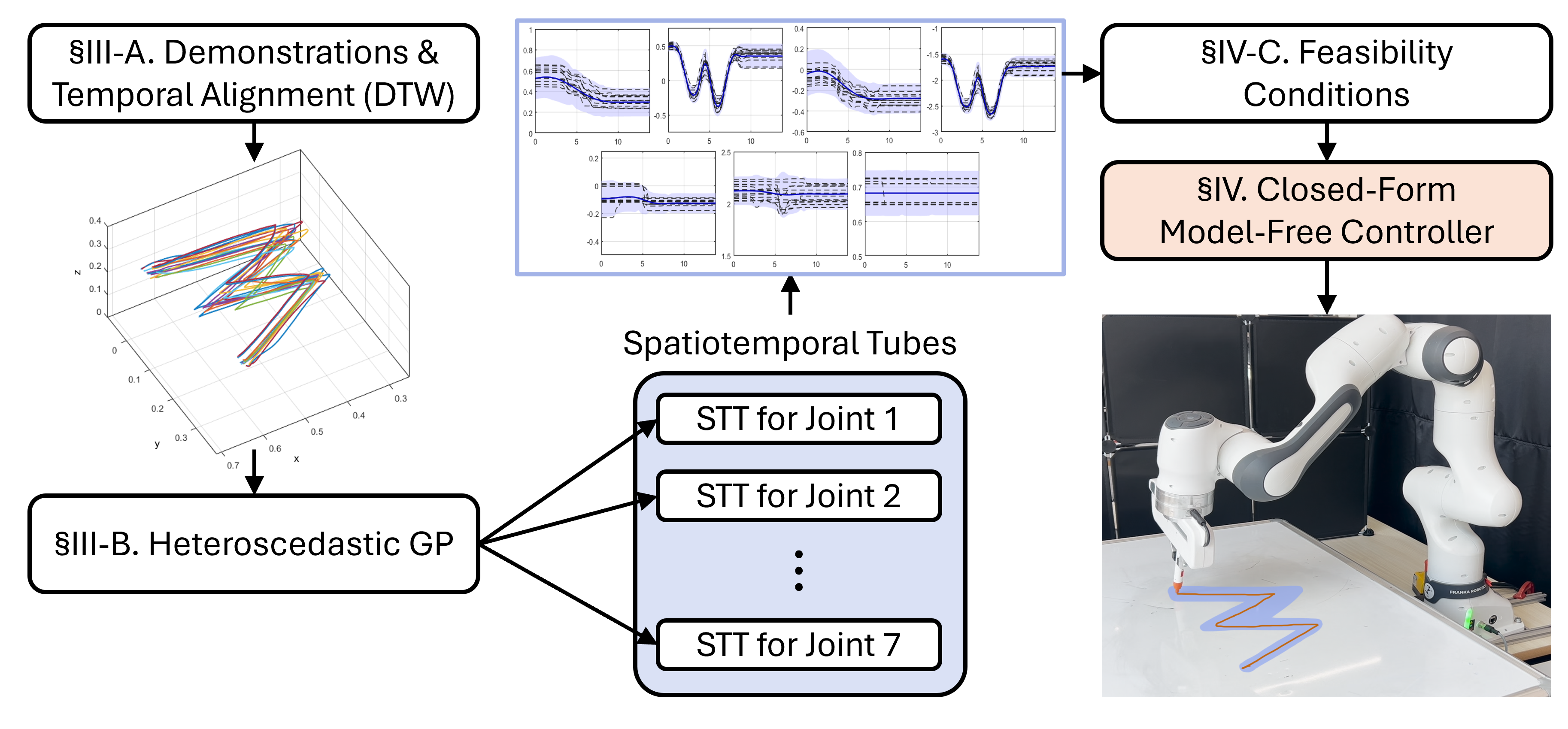}
    \caption{STT-LfD framework: Expert demonstrations are aligned using DTW and spatiotemporal tubes are obtained using HGPs. A closed-form controller then enforces tube invariance for safe and robust task execution under input constraints.
    }
    \label{fig:stt_lfd_pipeline}
\end{figure*}
The main contributions of this letter are summarized as follows:
{
\begin{enumerate}
\item 
\textbf{Unified Task Specification:} 
We propose a Learning from Demonstration framework that learns Spatiotemporal Tubes (STTs) directly from expert demonstrations using heteroscedastic Gaussian Processes. The learned STTs capture the time-varying precision requirements of the task and define admissible task executions as time-varying regions rather than a single reference trajectory.
\item 
\textbf{Closed-Form Control under Input Constraints:} 
We develop a closed-form control strategy that guarantees invariance of the learned STTs while respecting actuator limits. The proposed framework does not require explicit system identification or online optimization, making it computationally efficient for real-time robotic systems.
\item 
\textbf{Experimental Validation:} 
We validate the proposed framework through hardware experiments on a mobile robot and a 7-DOF FR~3 manipulator \cite{FRANKA}. The results show significant improvement in tracking error under disturbances (externally applied jerks and variations in mass matrix) and a three-order-of-magnitude reduction in computation time compared to existing pipelines.
\end{enumerate}
}



\section{Preliminaries and Problem Formulation}
\label{sec:prelim}

\textit{Notation.}
For $a,b\in\N$ and $a\leq b$, $[a;b]$ denotes a close interval in $\N$. 
The element-wise absolute value of a vector $x \in \R^n$ is denoted by $|x| := \col(|x_1|, \ldots, |x_n|)$ and the Euclidean norm using $\|x\|$. 
For $x, y \in \R^n$, the vector inequalities, $x \preceq y$ (and $x \succeq y$) represents $x_i \leq y_i$ (and $x_i \geq y_i$), $\forall i \in [1;n]$.
$x \uparrow (\downarrow) \ a$ indicates $x$ approaches $a$ from the left (right) side.
All other notation follows standard mathematical conventions.

\subsection{System Definition}
Consider an Euler-Lagrange (EL) system $\mathcal{S}$ described as:
\begin{align}
    \mathcal{S}: M(x)\Ddot{x} + V(x,\dot{x}) + G(x) = \tau + w(t), \label{eqn:sysdyn}
\end{align}
where $x(t) = [x_1(t), \ldots, x_n(t)]^\top \in X \subset \mathbb{R}^n$ is the state, $\tau(t) \in \mathbb{R}^n$ is the control input, and $w(t) \in \mathbb{W} \subset \R^n$ is an unknown external disturbance. $M(x) \in \R^{n \times n}$ denotes the mass matrix, $V(x, \dot{x}) \in \R^n$ represents the Coriolis and centrifugal terms, and $G(x) \in \R^n$ is the gravity vector. For brevity, we omit arguments and parentheses when the functional dependence is clear; for example, $M(x), V(x,\dot{x}), G(x)$ and $w(t)$ are represented as $M,V,G$ and $w$, respectively.

\begin{assumption}
    The mass matrix $M(x)$, the Coriolis and centrifugal terms $V(x, \dot{x})$, the gravity vector $G(x)$ and the external disturbance $w(t)$ are all unknown.
\end{assumption}

For the EL system $\mathcal{S}$, given control input bounds $\overline{\tau} \in \R^n$: $|\tau(t)| \preceq \overline{\tau},$ for all $t \in \R_0^+$, there are corresponding bounds on the norms of the system parameters 
\cite{ELbounds}. Although disturbance $w(t)$ and system parameters $M(x), V(x,\dot{x}),$ and $G(x)$ are unknown, their boundedness can be utilized effectively for control synthesis. To facilitate analysis, we adopt the following assumptions concerning the parameters of the EL system:

\begin{assumption}\label{assum_d}
    The external disturbance $w(t)$ satisfies $-\overline{w} \preceq w(t) \preceq \overline{w},$ for all $t \in \R_0^+$, where $\overline{w} \in \R^n$ is a known bound.
\end{assumption}

\begin{assumption}\label{assum_tau}
    Given the control bound $\overline{\tau}$, there exists a positive constant $\underline{m} \in \R$, such that $\underline{m}\overline{\tau} \preceq M^{-1}\overline{\tau}$.
\end{assumption}

\begin{assumption}\label{assum_V}
The Coriolis and centrifugal terms $V$ and the gravity vector $G$ satisfy $\underline{V}_M \preceq V_M \preceq \overline{V}_M$, where $V_M := -M^{-1}(V+G)$ and $\underline{V}_M, \overline{V}_M \in \R^n$.
\end{assumption}

\begin{assumption}\label{assum_Md}
    The inverse of the mass matrix scales the disturbance as $-\|M^{-1}\| \overline{w} \preceq M^{-1}w \preceq \|M^{-1}\| \overline{w}$. This implies, there exists $\underline{m}_i \in \R^+$, such that $-\underline{m}_i \overline{w} \preceq M^{-1}w \preceq \underline{m}_i \overline{w}$.
\end{assumption}

The Assumptions \ref{assum_d}-\ref{assum_Md} provide the bounds on various system parameters, which are utilized for establishing the feasibility conditions in Section \ref{sec:feas}.

\subsection{Problem Statement}
The goal is to develop a unified framework that converts expert demonstrations into a time-varying intent envelope, expressed as a spatiotemporal tube (STT), and synthesizes a control law that safely executes the task.
\begin{problem}
Given a set of expert demonstrations over a finite time horizon $\mathcal{T}$, learn an STT $\Gamma(t) \subset \mathbb{R}^n$ and design a  closed-form control input $\tau(t)$ such that the system trajectory remains inside $\Gamma(t)$ for all $t \in \mathcal{T}$, the control input satisfies the actuator bounds $|\tau(t)| \preceq \overline{\tau}$, and the task is executed robustly despite unknown dynamics and external disturbances $w(t)$.
\end{problem}

\textit{Solution Sketch:}
To solve the problem, we propose a unified Learning from Demonstration framework that combines data-driven motion representation with feedback control. From multiple expert demonstrations, we first learn a compact, time-indexed representation of the task that captures both the nominal behavior and variability across demonstrations. This representation is expressed as a \emph{Spatiotemporal Tube (STT)}, which defines an admissible, time-varying envelope within which any system trajectory is considered a valid execution of the demonstrated task. 
We then design a computationally efficient feedback controller that constrains the system trajectory within the learned tube while respecting actuator limits, thereby ensuring safe, robust, and consistent task reproduction without explicit system identification. A schematic overview of the proposed solution is shown in Fig.~\ref{fig:dtw_stt_pipeline}.

\section{Learning Spatiotemporal Constraints}
Here, we present the Spatiotemporal Tubes-based Learning from demonstration (STT-LfD) framework for learning skills from human demonstrations using Gaussian Process \cite{seeger2004gaussian}. 

{STTs have previously been developed in the context of temporal logic specifications \cite{das2025spatiotemporal, das2025real}, where the tubes are designed through optimization or neural networks to satisfy formal task constraints. In contrast, this work uses the STT framework in an LfD setting, where the tubes are learned directly from demonstrations and capture both the nominal task behavior and the variability across demonstrations.}

To achieve this, we use heteroscedastic Gaussian Processes to synthesize time-varying STTs that represent admissible task executions. The learned tubes naturally encode regions requiring high precision as narrow tube segments and more flexible task phases as wider regions.

We begin by formalizing the data collection procedure, with particular attention to the temporal alignment of demonstrated trajectories, before introducing STTs as a Gaussian process–based representation of the learned skill.

\subsection{Data Collection and Temporal Alignment} \label{sec:data}
In LfD, we assume access to multiple expert demonstrations of the same task obtained by recording a human operator guiding the robot through the motion. Each demonstration consists of time-series measurements of the robot state sampled at a fixed rate. 
{The demonstrations are assumed to correspond to dynamically feasible executions of the task under the target platform's actuator limits, so that the learned STTs represent realizable task specifications.}
Furthermore, although repeated executions typically follow the same general path, natural human variability leads to differences in execution speed, resulting in temporal misalignment across demonstrations. Since the proposed framework learns STTs as time-dependent task representations, ensuring temporal consistency among demonstrations is therefore important. 

\begin{figure}[t]
\centering
\includegraphics[width=0.68\textwidth]{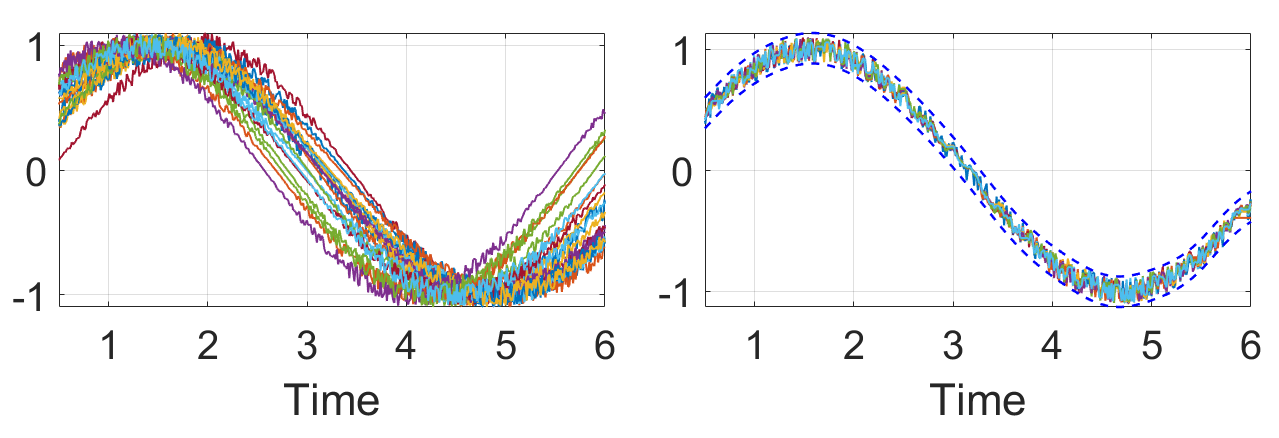}
\caption{Temporal alignment of demonstrations. (a) Raw demonstrations showing natural human timing variability, (b) DTW-aligned with the learned STT, where the time-varying width captures narrow precise and wider flexible segments.}
\label{fig:dtw_stt_pipeline}
\end{figure}

To address this, we use Dynamic Time Warping (DTW) \cite{senin2008dynamic, arduengo2023gaussian} as a preprocessing step. DTW aligns demonstrations by nonlinearly warping their time axes, allowing trajectories that follow similar spatial paths but are executed at different speeds to be expressed on a common time grid, while preserving the overall task geometry.
The resulting temporally aligned dataset is then used to learn the STTs described in the next subsection. An illustrative example of raw demonstrations and their DTW-aligned counterparts is shown in Fig.~\ref{fig:dtw_stt_pipeline}.

\subsection{HGP-based Spatiotemporal Tube Construction}
Given a set of temporally aligned demonstrations (Section \ref{sec:data}), we construct the STTs that compactly represents both the nominal trajectory and the variability observed across demonstrations. For each state dimension $j \in [1;n]$, we model the desired time-varying behavior of the state using heteroscedastic GP (HGP) \cite{le2005heteroscedastic} as
\begin{align}
    x_j(t) \sim \mathcal{GP}(\mu_j(t), \sigma_j^2(t)),
\end{align}
where $\mu_j(t)$ and $\sigma_j(t)$ are the posterior mean and standard deviation inferred from the demonstrations.

Unlike standard GPs, which assume a constant (homoscedastic) noise variance across time, HGPs \cite{heteroscedastic} explicitly allow the uncertainty to vary with time. This is particularly important in the intent envelope that captures task-critical precision, where different phases of a task can show very different levels of variability. For example, precise contact-rich motions often require tight tracking, while free-spaces allow more freedom. By modeling time-dependent uncertainty, HGPs allow us to naturally capture this non-uniform variability and provide more informative, task-aware tube design.

The one-dimensional STT for the $j$-th state is defined as:
\begin{equation}\label{eqn:stt}
    \Gamma_j(t) := \left[\mu_j(t) - \lambda \sigma_j(t), \; \mu_j(t) + \lambda \sigma_j(t)\right]^\top,
\end{equation}
where $\lambda \in \R^+$ is a user-defined scaling parameter for the tube width.
By taking the cartesian product of these intervals across all state dimensions, we obtain the full $n$-dimensional STT
    $\Gamma(t) := \prod_{j=1}^{n} \Gamma_j(t) \subset \mathbb{R}^n.$
This forms an axis-aligned hyperrectangle at each time $t$ that defines the spatiotemporal envelope within which the system state is expected to remain to accurately imitate the demonstrated task.

\begin{remark}[Tube coverage probability]\label{rem:tubecovprob}
Given the tube width parameter $\lambda_j \in \R^+$, the probability that the demonstrations are contained within the STT can be lower bounded as
$\mathbb{P} \left\{ \cap_{j=1}^{n} \mu_j(t) \!-\!\lambda_j\sigma_j(t) \!\leq\! x_j(t) \!\leq \! \mu_j(t) \!+\! \lambda_j\sigma_j(t), \forall t \in \mathcal{T} \right\} \!\geq\! (1 \!-\! \varepsilon)^n,$
where $\varepsilon \in (0,1)$ determines the confidence level. 
$\lambda_j \!:=\! \sqrt{2 \|x_j\|^2_{k_j} \!+\! 300 \gamma_j \log^3\!\left(\tfrac{N+1}{\varepsilon}\right)},$
with $\gamma_j$ denoting the information gain and $\|x_j\|_{k_j}$ a kernel norm measuring function complexity. Computing these quantities exactly is generally NP-hard; therefore, in practice, we estimate tube coverage probabilities using Monte Carlo sampling. For simplicity, we also set $\lambda_j \!=\! \lambda, \forall j \in [1;n]$.
Increasing $\lambda$ increases the probability of capturing all demonstrations but results in wider tubes, reducing tracking precision. Moreover, achieving higher confidence or narrower probability intervals typically requires more demonstrations and increased computation. A more detailed discussion can be found in \cite{srinivas2012information, jagtap2020control}.

Importantly, this probabilistic statement applies only to the learning phase, where the tube is constructed from finite and noisy data. Once learned, the STT represents a fixed underapproximation of the set of geometrically similar demonstrated trajectories and is therefore treated as a deterministic, time-varying constraint set. The control law in Section~\ref{sec:controller} then enforces invariance of this set deterministically, resulting in formal task execution guarantees.
\end{remark}

We now state the theorem that formally establishes that any trajectory remaining inside the tube can be considered a consistent execution of the demonstrated task.

\begin{theorem}
Let $\Gamma(t)$ be the STT constructed from multiple demonstrations of a target task. If a system trajectory $x(t) \in \mathbb{R}^n$ satisfies $x(t) \in \Gamma(t)$ for all $t \in \mathcal{T}$, then the task is completed consistently with the demonstrations, up to bounded deviations governed by the tube width parameter $\lambda$.    
\end{theorem}
\begin{proof}
Let $x(t) \in \mathbb{R}^n$ be the system trajectory and $\Gamma(t) := \prod_{j=1}^n [\mu_j(t) - \lambda \sigma_j(t), \mu_j(t) + \lambda \sigma_j(t)]$ the STT learned from $N$ aligned demonstrations.
For each dimension $j \in [1;n]$ and all $t \in \mathcal{T}$, the condition $x(t) \in \Gamma(t)$ implies 
$|x_j(t) - \mu_j(t)| \leq \lambda \sigma_j(t).$
Since $\mu_j(t)$ and $\sigma_j(t)$ are estimated from the distribution of the demonstrations using GPR, the interval $[\mu_j(t) \pm \lambda \sigma_j(t)]$ captures the range of variation observed across demonstrations, with confidence level determined by $\lambda$.
Therefore, if $x(t) \in \Gamma(t)$ for all $t \in \mathcal{T}$, then $x(t)$ remains within the demonstrated task envelope at all times, implying that it satisfies the task up to the allowable deviation~$\lambda$.
\end{proof}


The property that any trajectory remaining within the STT is consistent with the demonstrated task, up to bounded deviations, makes STTs a powerful representation for imitation learning. This allows robust generalization while preserving formal task guarantees. Next, we present the control law that constrains the system state to remain within the learned STTs.

\section{Controller Synthesis}\label{sec:controller}

{In Learning from Demonstration, respecting actuator constraints is essential for safe hardware deployment and reliable task reproduction. However, previous STT-based control formulations did not explicitly account for input constraints and could lead to unbounded control actions. Motivated by this limitation, we develop a closed-form control strategy that guarantees forward invariance of the learned STTs while respecting prescribed input constraints.}

The controller is devised in a two-step approach, a velocity-level control \eqref{eqn:sysDyn_vel}, and an acceleration-level control \eqref{eqn:sysDyn_acc}.
\begin{subequations}
  \begin{align}
    \dot{x} &= v, \label{eqn:sysDyn_vel} \\
    \dot{v} &= M(x)^{-1} \left(-V(x, v) - G(x) + u + w \right). \label{eqn:sysDyn_acc}
  \end{align}
\end{subequations}
The proposed two-step architecture is inspired by the funnel-control frameworks presented in \cite{PPCfeedback} and \cite{hard_soft}, where related backstepping-like methodologies were developed for pure-feedback and Euler--Lagrange systems. {However, unlike classical funnel-control approaches, the proposed framework develops a control strategy for enforcing forward invariance of learned STTs under actuator constraints.}

\subsection{Stage I: Velocity-level Control}
Given the STT $\Gamma_j(t)$, to enforce $x_j(t) \in \Gamma_j(t)$, for all $j \in [1;n]$ and $t \in \mathcal{T}$, define the reference velocity vector $v_r(x,t)$
\begin{equation}\label{eqn:velcon}
    v_r(x,t) = -\text{diag}(\Psi(\varepsilon_x))\overline{v},
\end{equation}
where $\varepsilon_x(x,t) := [\varepsilon_{x,1}(x_{1},t), \ldots, \varepsilon_{x,n}(x_{n},t)]^\top$, with $\varepsilon_{x,j}(x_j,t) = \big(x_j(t) - \mu_j(t)\big)/(\lambda\sigma_j(t))$ is the normalized error,
the map $\Psi:\R^n \rightarrow \R^n$ is given in Appendix \ref{sec:clamp}, and $\overline{v} \in \R^n$.


\subsection{Stage II: Acceleration-level Control}
To ensure smooth tracking of the reference velocity $v_{r}(t)$ from \eqref{eqn:velcon} in Stage I, we now bound the velocity error $e_v(t):= v(t) - v_{r}(t)$ within
exponentially decaying funnel constraints $\rho_v: \R_0^+ \rightarrow \R^n$, given by:
$\rho_v(t) = e^{-\mu_v t}(p_v-q_v) + q_v,$ as
\begin{equation}\label{eqn:fun2}
    -\rho_v(t) \prec e_v(t) \prec \rho_v(t).    
\end{equation}
Here $p_v \in \R^n$ represents the initial width of the funnel, satisfying $|e_v(0)| \preceq p_v$, $q_v \in \R^n$ specifies the ultimate bound on the velocity error, with $0^{n \times 1} \prec q_v \prec p_v$, and $\mu_v > \textbf{0}^{n \times n} \in \R^{n \times n}$ is a diagonal matrix, controlling the decay rate of the funnel constraint.
Next, we define the normalized velocity error $\varepsilon_v(t):= \text{diag}(\rho_{v})^{-1}e_v(t).$
The final control input $\tau(t)$ at the acceleration level is then formulated as:
\begin{equation}\label{eqn:con}
    \tau(t) = -\text{diag}(\Psi(\varepsilon_v(t)))\overline{\tau},
\end{equation}
where $\Psi:\R^n \rightarrow \R^n$ is the transformation function in Appendix \ref{sec:clamp}, and $\overline{\tau} \in \R^n$ is the maximum permissible torque.

Now, we introduce the feasibility conditions under which we proceed to guarantee that the system state evoles within the tube $x_j(t) \in \Gamma_j(t)$ for all $t \in \mathcal{T}$ and $j \in [1;n]$.

\subsection{Feasibility Condition}\label{sec:feas}
By constraining the state to remain within the tube bounds, the system achieves the demonstrated task. However, input constraints address practical concerns such as actuator safety. This results in a trade-off between performance and resource limitations. To effectively navigate this trade-off, we establish the following feasibility constraints.
\subsubsection{Stage I}
Given the tube $\Gamma(t) := \prod_{j=1}^n [\mu_j(t) - \lambda \sigma_j(t), \mu_j(t) + \lambda \sigma_j(t)]$ with bounds $|\dot\mu(t) \pm \lambda \dot\sigma(t)| \preceq \overline{\Gamma} = [\overline{\Gamma}_1, \ldots, \overline{\Gamma}_n]^\top$, the maximum permissible velocity $\overline{v}$ should satisfy the following constraint:
\begin{equation}\label{eqn:feas1}
    \overline{v} \succeq \overline{\Gamma} + p_v.
\end{equation}
This condition ensures that the system has sufficient actuation capacity to reproduce the maximum speed of the demonstrations $\overline{\Gamma}$ and reject the worst-case perturbation ($|e_v| \prec p_v$).

\subsubsection{Stage II}
Given funnel constraints $\rho_v(t) = e^{-\mu_vt}(p_v-q_v)+q_v$, and system dynamics in \eqref{eqn:sysdyn} with Assumptions \ref{assum_d}- \ref{assum_Md}, the maximum permissible torque $\overline{\tau}$ should adhere to the following constraint:
\begin{equation}\label{eqn:feas2}
    \overline{\tau} \succeq \frac{1}{\underline{m}} \left( \max(-\underline{V_M}, \overline{V_M}) \!+\! \underline{m}_i\overline{d} \!+\! \mu_v(p_v-q_v) \!+\! \overline{a}_r \right),
\end{equation}
where given the definition of bounded transformation function in Appendix \ref{sec:clamp}, {we select an upper bound $\overline{a}_r \in \R^n$, such that $|\dot{v}_r| = |\overline{v} \ \text{diag}(\Psi(\varepsilon_x))| \preceq \overline{a}_r$. For example, if $\Psi_i(s) := \tanh(5s)$ for all $i \in [1;n]$, then we can choose $\overline{a}_r = 5 \overline{v}$.}

The theorem formally summarizes the approximation-free feedback controller proposed in this paper.

\begin{theorem}\label{thm:bdcontrol}
    Consider the Euler-Lagrange system $\mathcal{S}$ in \eqref{eqn:sysdyn} with assumptions \ref{assum_d}-\ref{assum_Md} and the STTs $\Gamma(t)$ obtained from the demonstrations. If the initial state error $e_x(0)$ satisfies $e_x(0) \in \Gamma(0)$, the initial velocity error follows $|e_{v}(0)| < p_v$, and if feasibility conditions \eqref{eqn:feas1} and \eqref{eqn:feas2} hold, then the closed-form controller $\tau(t)$ in \eqref{eqn:con} guarantees
    \begin{itemize}
        \item[(i)] The system state $x(t)$ evolves within the STTs:
        $x_j(t) \in \Gamma_j(t) \ \forall t \in \R_0^+, \ j \in [1;n]$ and
        \item[(ii)] The control input $\tau(t)$ is bounded within the prescribed limit:
        $|\tau(t)| \preceq \overline{\tau}, \ \forall t \in \R_0^+.$
    \end{itemize}
\end{theorem}

\begin{figure*}[t]
    \centering
    \includegraphics[width=\textwidth, height=0.25\linewidth]{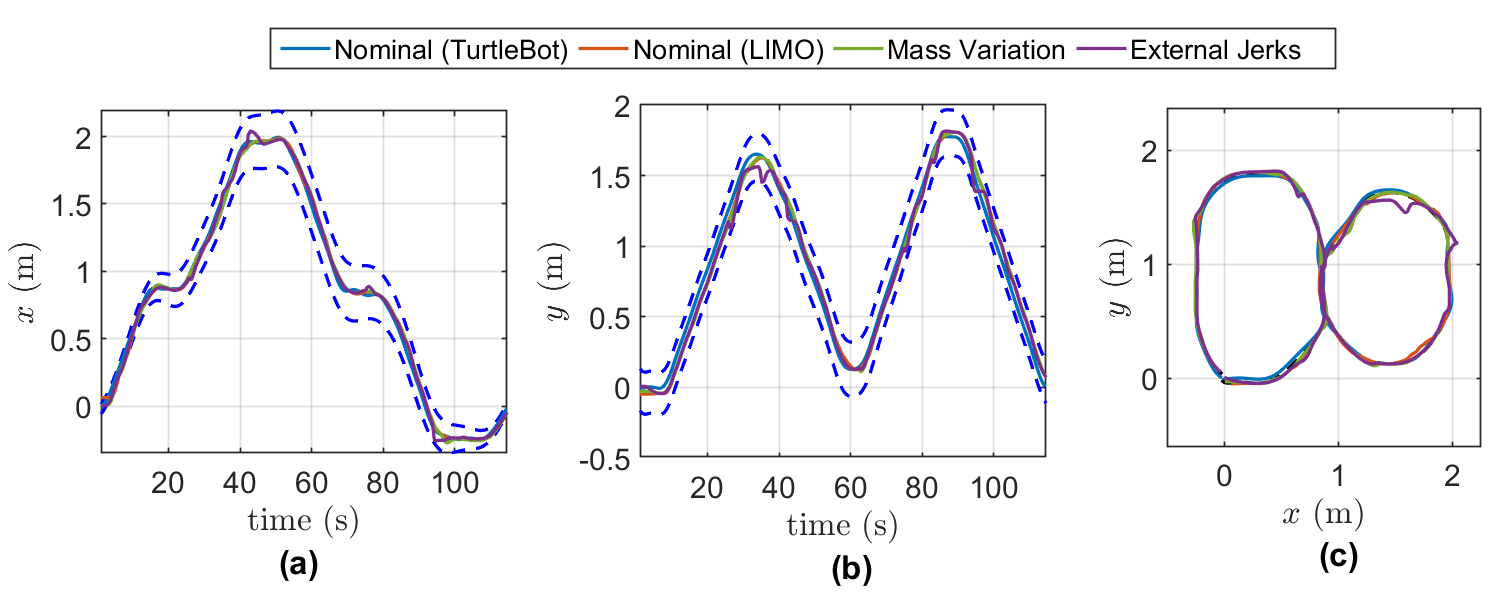}
    \caption{{Mobile robot experiments under varying platforms and disturbances. 
    (a)-(b) The synthesized STTs for x and y dimensions. 
    (c) Resulting figure-eight trajectories for four experimental conditions: nominal TurtleBot execution, AgileX LIMO execution, mass variation, and external jerks.} }
    \label{fig:omni}
\end{figure*}

\begin{figure}[htbp]
    \centering
    \includegraphics[width=0.6\linewidth]{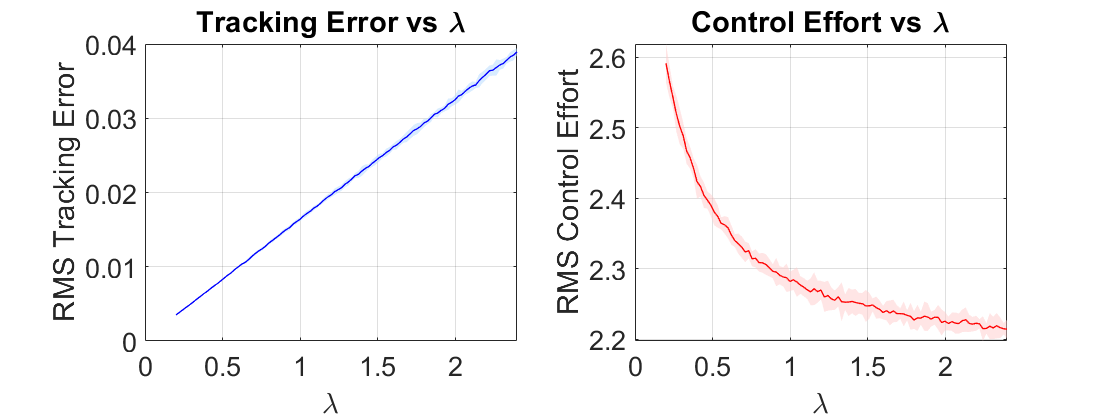}
    \caption{{Sensitivity study of the STT width parameter $\lambda$ for the mobile robot figure-eight trajectory task. 
    (left) RMS tracking error and
    (right) RMS control effort as a function of $\lambda$. The solid curves denote the mean over 10 trials with bounded random disturbances, while the shaded regions represent ±2 standard deviations.}}
    \label{fig:lambda_sensitivity}
\end{figure}

\begin{proof}
    The proof is divided into two stages:

    \textbf{Stage I.} In stage I, we prove that the reference velocity vector $v_r(t)$ in \eqref{eqn:velcon} enforces the state $x(t)$ to remain in the STT $\Gamma(t)$ for all time $t \in \mathcal{T}$. 

    We will prove this via contradiction. Let $\tx$ be the first time instance when the state $x(t)$, on the application of velocity input $v_r(t)$ \eqref{eqn:velcon}, exits the STT $\Gamma(t)$, i.e., $\exists j \in [1;n],$
    $$x_j(\tx) \leq \mu_{j}(\tx)-\lambda \sigma_{j}(\tx) \text{ or } x_j(\tx) \geq \mu_{j}(\tx)+\lambda \sigma_{j}(\tx).$$
    Then, for all $(t,j) \in [0, \tx) \times [1;n],$
\begin{gather}\label{Eq:inqe_tx}
    \mu_{j}(t)-\lambda \sigma_{j}(t) < x_j(t) < \mu_{j}(t)+\lambda\sigma_{j}(t).
\end{gather}
We will consider the following two cases for $t \in [0,\tx)$.

\textbf{Case I.} There exists $j \in [1;n]$ such that $x_j(t)$ approaches the upper tube constraint, i.e., $x_j(t) \rightarrow \mu_{j}(t)+\lambda\sigma_{j}(t) \implies \underline{\delta}_j(t) := x_j(t) - \mu_{j}(t)-\lambda\sigma_{j}(t) \rightarrow 0$. Following \eqref{Eq:inqe_tx}, we have
\begin{align*}
    &x_j(t) \!<\! \mu_{j}(t)\!+\!\lambda\sigma_{j}(t) \Rightarrow \underline{\delta}_j(t) \uparrow 0 
    \Rightarrow \lim_{\underline{\delta}_j(t) \uparrow 0} \frac{d}{dt} \underline{\delta}_j(t) \!>\! 0
    \Rightarrow \lim_{\underline{\delta}_j(t) \uparrow 0} \dot{x}_{j}(t) \!>\! \lim_{\underline{\delta}_j(t) \uparrow 0} \dot{\mu}_{j}(t)\!+\!\lambda\dot{\sigma}_{j}(t) 
     \!>\! -\overline{\Gamma}_j.
\end{align*}

Now, the actual velocity in \eqref{eqn:sysDyn_vel} is given by $\dot{x}_j(t) = v_{r,j}(t) + e_{v,j}(t)$, where $v_{r,j}(t)$ is the reference velocity law in \eqref{eqn:velcon}, and $e_{v,j}(t) < p_{v,j}$ (we ensure this in Stage II). Therefore, as $(x_j(t) - \mu_{j}(t)-\lambda\sigma_{j}(t)) \rightarrow 0$, $v_{r,j}(t) \rightarrow -\overline{v}_{j}$, and we obtain that there exists $j \in [1;n]$, such that:
\begin{align*}
    &\lim_{\underline{\delta}_j(t) \uparrow 0} \dot{x}_{j}(t) = -\overline{v}_{j} + e_{v,j}(t) > -\overline{\Gamma}_j 
    \implies -\overline{v}_{j} + p_v > -\overline{\Gamma}_j 
    \implies \overline{v}_{j} < \overline{\Gamma}_j + p_v.
\end{align*}
which contradicts \eqref{eqn:feas1}.

\textbf{Case II.} There exists $j \in [1;n]$ such that $x_j(t)$ approaches the lower tube constraint, i.e., 
$x_j(t) \rightarrow \mu_{j}(t)-\lambda\sigma_{j}(t)$. 
Using the lower-bound feasibility condition leads to the same contradiction and the proof follows analogously. Therefore the lower tube boundary cannot be crossed.



Therefore, the reference velocity vector $v_r$ in~\eqref{eqn:velcon} constrains $x(t)$ within the tubes for all time, i.e., $\mu_{j}(t)-\lambda\sigma_{j}(t) < x_j(t) < \mu_{j}(t)+\lambda\sigma_{j}(t), \forall t \geq 0$.


\textbf{Stage II.} In stage II, we prove that the control law $\tau(t)$ in \eqref{eqn:con} keeps the velocity tracking error $e_v(t)$ within the funnel $[-\rho_v(t), \rho_v(t)]$ for all $t \in \R_0^+$ \eqref{eqn:fun2}.

We will prove this via contradiction. Let $\tx$ be the first time instance when the velocity error $e_v(t)$, on the application of input $\tau(t)$ \eqref{eqn:con}, violates \eqref{eqn:fun2}, i.e., $\exists j \in [1;n]$,
    $$e_{v,j}(\tx) \leq -\rho_{v,j}(\tx) \text{ or } e_{v,j}(\tx) \geq \rho_{v,j}(\tx).$$
Then, for all $(t,j) \in [0, \tx) \times [1;n],$
\begin{gather}\label{Eq:inqe_tv}
    -\rho_{v,j}(t) < e_{v,j}(t) < \rho_{v,j}(t).
\end{gather}
We will consider the following two cases for $t \in [0,\tx)$.

\textbf{Case I.} For some $j \in [1;n]$, $e_{v,j}(t)$ approaches the upper funnel constraint, i.e., $e_{v,j}(t) \rightarrow \rho_{v,j}(t) \implies e_{v,j}(t) - \rho_{v,j}(t) =: \overline{\delta}_{v,j} \rightarrow 0$. Following \eqref{Eq:inqe_tv}, we have:
\begin{align*}
    &e_{v,j}(t) < \rho_{v,j}(t) \implies \overline{\delta}_{v,j} \uparrow 0 \implies \lim_{\overline{\delta}_{v,j} \uparrow 0} \frac{d}{dt} \overline{\delta}_{v,j} > 0 \\
    \implies &\lim_{\overline{\delta}_{v,j} \uparrow 0} \dot{e}_{v,j}(t) > \lim_{\overline{\delta}_{v,j} \uparrow 0} \dot{\rho}_{v,j}(t) > -\mu_{v,j}(p_{v,j}-q_{v,j}) \\
    \implies &\lim_{\overline{\delta}_{v,j} \uparrow 0} \dot{v}_{i}(t) > -\mu_{v,j}(p_{v,j}-q_{v,j}) + \dot{v}_r(t).
\end{align*}
Therefore, there exists $i \in [1;n]$, such that 
\begin{gather}\label{eqn:dv_b1}
    \lim_{\overline{\delta}_{v,j} \uparrow 0} \dot{v}_{i}(t) > -\mu_{v,j}(p_{v,j}-q_{v,j}) - \overline{a}_{r,i}.
\end{gather}
Since, $\lim_{\overline{\delta}_{v,j} \uparrow 0} \varepsilon_{v,j}(t) = 1$, we obtain $\lim_{\overline{\delta}_{v,j} \uparrow 0} \tau_{i}(t) = -\overline{\tau}_{i}$.
Using the dynamics~\eqref{eqn:sysDyn_acc} and feasibility condition \eqref{eqn:feas2}
\begin{align*}
    \lim_{\overline{\delta}_{v,j} \uparrow 0} \!\!\dot{v}_{i}(t) \!\leq\! \overline{V}_{M,i} \!-\! \underline{m} \overline{\tau}_i \!+\! \underline{m}_i\overline{d}_i
    \!\leq\! \!-\mu_{v,j}(p_{v,j} \!-\! q_{v,j}) \!-\! \overline{a}_{r,i}
\end{align*}
which contradicts \eqref{eqn:dv_b1}. 

\textbf{Case II.} For some $i \in [1;n]$, $e_{v,j}(t)$ approaches the lower funnel constraint, i.e., 
$e_{v,j}(t) \rightarrow -\rho_{v,j}(t)$.
Using the lower-bound feasibility condition leads to the same contradiction and the proof follows analogously. Therefore the lower funnel boundary cannot be crossed.
Therefore, the controller $\tau$ in~\eqref{eqn:con} constrains the velocity error within the funnel, i.e.,
$-\rho_v(t)\prec e_v(t)\prec \rho_v(t), \forall t\ge0.$

Combining Stages~I and~II, the control input $\tau(t)$ in \eqref{eqn:con}, under feasibility conditions \eqref{eqn:feas1} and \eqref{eqn:feas2}, ensures that the system state $x(t)$ evolves within the STTs $\Gamma(t)$ in \eqref{eqn:stt}.
\end{proof}

\begin{remark}
In Stage~I of the proof, we do not assume perfect tracking of the virtual velocity. Instead, we consider the actual closed-loop dynamics $\dot{x}(t) = v_r(t) + e_v(t),$ where $e_v(t) = v(t) - v_r(t)$ is the velocity tracking error. The analysis in Stage~I shows that the $x(t) \in \Gamma(t)$ as long as $|e_v(t)| < \rho_v(t)$ and the feasibility condition~\eqref{eqn:feas1} holds. Stage~II then completes the proof by showing that the control input $\tau(t)$ in~\eqref{eqn:con} enforces the bound $|e_v(t)| < \rho_v(t)$ for all time, which in turn guarantees that the state $x(t)$ evolves within the tube $\Gamma(t)$.
\end{remark}

\begin{remark}
{Assumptions~\ref{assum_d}-\ref{assum_Md} and the associated feasibility conditions~\eqref{eqn:feas1}-\eqref{eqn:feas2} are introduced only to establish theoretical guarantees on forward invariance of the learned STTs and bounded control execution. By “unknown dynamics” or “model-free execution,” we specifically mean that the proposed framework does not require explicit identification of the system parameters, such as the inertia matrix, Coriolis terms, or gravitational dynamics, for either STT generation or controller implementation. The learned STTs are constructed directly from demonstrations, and the control law~\eqref{eqn:con} can be implemented without explicit parametric system identification.
}
\end{remark}

\begin{figure*}
    \centering
     \includegraphics[width=\textwidth]{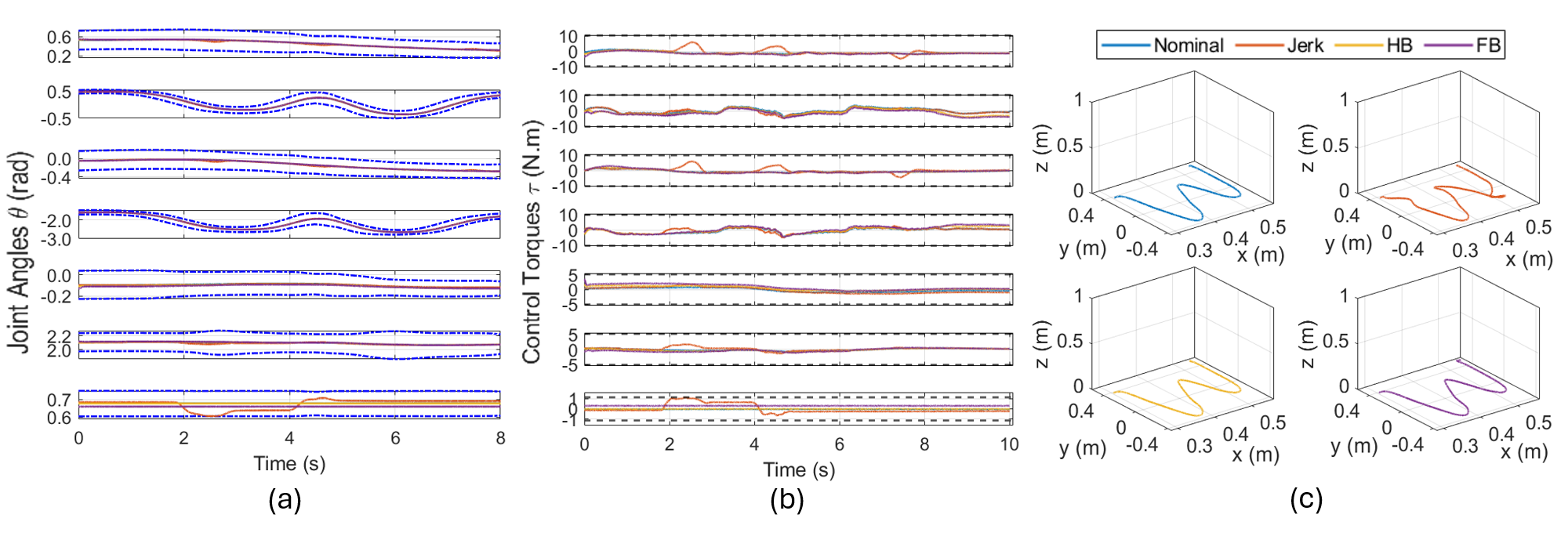}
    \caption{7-DOF manipulator under nominal conditions, jerk disturbances, and mass variations:
    (a) Joint-space trajectories remain within the learned STT.
    (b) Corresponding control torques generated by STT-LfD stay within the prescribed bounds while adapting to disturbances.
    (c) End-effector execution of the demonstrated task, showing consistent tracking.}    
    \label{fig:frankaW}
\end{figure*}


\begin{figure}[t]
    \centering
    \includegraphics[width=0.65\textwidth]{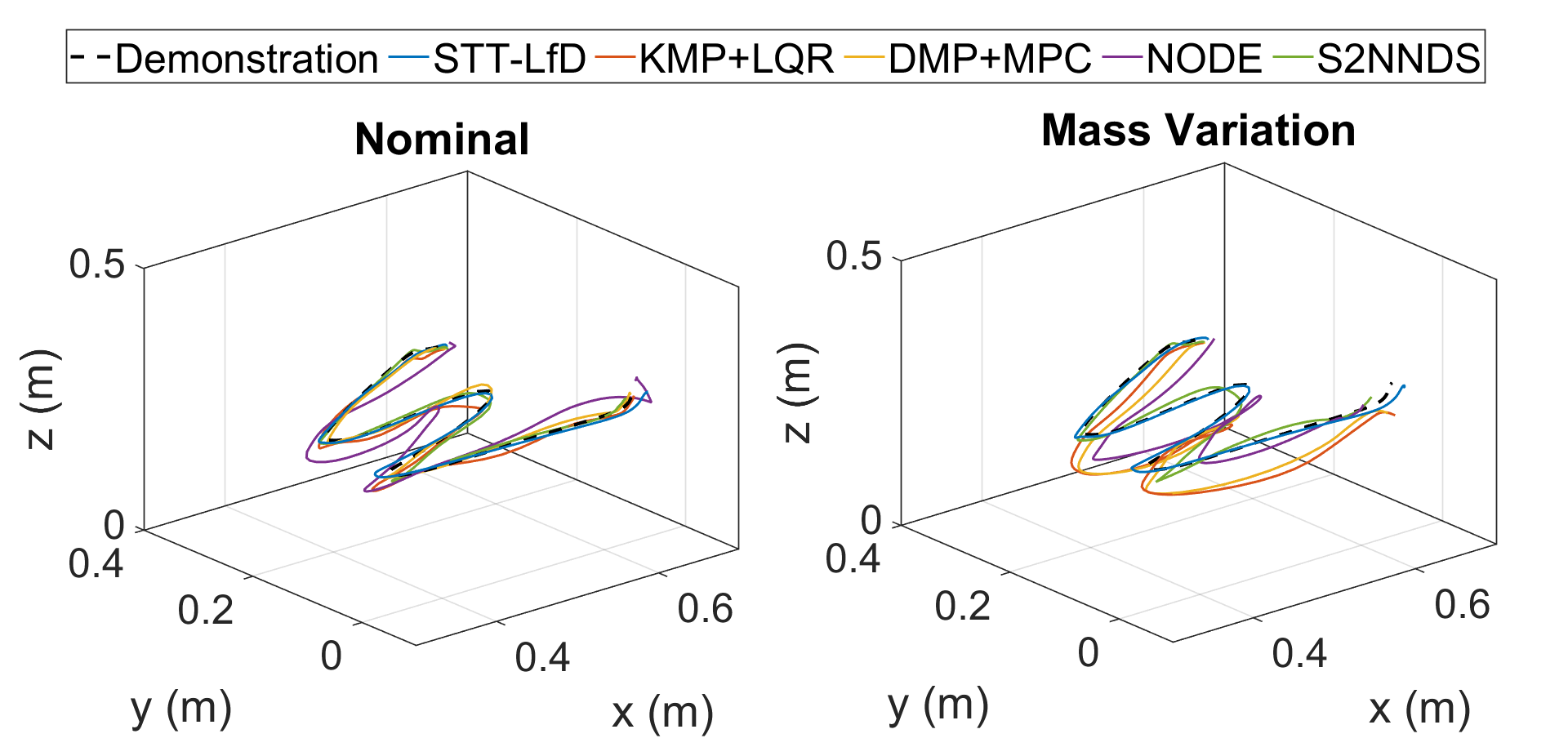}
    \caption{Comparison of End Effector Trajectory of 7-DOF Manipulator with other baselines, showing robustness to mass variation. Left: Nominal. Right: with added mass.}
    \label{fig:compare_EE}
\end{figure}

\section{Experimental Validation}\label{sec:validation}
We demonstrate the proposed STT-LfD framework through hardware experiments on two platforms: an omnidirectional mobile robot and a 7-DOF Franka Research~3 manipulator. The experiments focus on two aspects: handling spatiotemporal ambiguity in complex tasks and maintaining robustness against unknown dynamics and disturbances.

\subsection{Task Design:}
The tasks were chosen to highlight spatiotemporal aspects that are difficult for conventional methods. \\
\textbf{Mobile Base (Figure ``8" Tracking):}
An omnidirectional mobile robot follows a self-intersecting figure ``8" trajectory (Fig.~\ref{fig:omni}). At the crossing point, purely spatial methods cannot distinguish between the two distinct phases of the motion.The STT resolves this ambiguity through temporal indexing, allowing the robot to maintain the correct task sequence. \\
\textbf{7-DOF Manipulator (Alphabet ``W" Drawing):}
The manipulator performs a kinesthetic demonstration of drawing the letter ``W" (Fig.~\ref{fig:frankaW}). The task includes both precise corner points and flexible connecting strokes. Unlike decoupled trackers, the HGP-based STT naturally extracts these varying precision requirements from human variability, providing a time-varying constraint set for the controller.

\subsection{Implementation Details:}
The framework was implemented in C++ using the Libfranka API for the manipulator and a ROS2 interface for the mobile robot. All experiments were conducted on a system with an Intel Core i7 CPU at 3.6~GHz and 16~GB RAM.

We collected 15 expert demonstrations with variations in speed and execution style while maintaining the same overall path. The STT scaling parameter was set to $\lambda = 1.2$, representing an intent envelope that captures approximately $77\%$ of the expert demonstration variance. This choice provides a balanced trade-off between accurate trajectory tracking and robustness to sensor noise. A smooth saturation function $\Psi(s) = \tanh(5s)$ was used to limit control effort near the tube boundaries. 
For the mobile robot, the velocity and acceleration limits were set to $0.1~m/s$ and $0.1~m/s^2$, while the manipulator joint velocity and torque limits were selected according to \cite{FRANKA}.

\subsection{Execution Behavior and Tube Invariance:}
The STTs clearly capture the intent of the experts, reflecting how demonstrations vary across different parts of the task.

In the figure ``8'' experiment, the demonstrations are highly consistent near the initial and final regions, which results in narrower tubes towards the start and end. As the robot moves through the larger regions of the track, the demonstrations show greater variability in speed and spatial execution, leading to wider tube sections that allow more flexibility. The temporal indexing provided by the STT ensures that the robot passes through the intersection with the correct motion direction, as illustrated by the task-space trajectory. Throughout the experiment, the executed trajectories remain inside the learned tubes and accurately trace the figure ``8'' trajectory while the control inputs remain bounded.
{To further evaluate robustness to variations in system dynamics and external disturbances, we additionally reproduced the same learned STTs using an AgileX LIMO robot under same input constraints without modifying the controller structure or requiring explicit system identification. We also evaluated the framework under a 1\,kg load variation and manually applied external jerky dynamic disturbances during execution. Across all tested conditions, the proposed controller successfully maintained the system trajectory within the learned STTs while completing the task, as shown in Fig.~\ref{fig:omni}.
}

A similar pattern is observed in the ``W'' drawing task. Since all demonstrations begin from nearly the same configuration, the STTs are initially narrow. The tubes widen along the straight strokes where human motion varies more and tighten again near sharp corners that require precise execution. To evaluate robustness under changing dynamics, we additionally performed experiments with a varying mass matrix and manually applied external jerks during execution. Despite these disturbances and dynamic variations, the proposed controller successfully modulated the torques to maintain the manipulator state within the learned STT, as illustrated in Fig.~\ref{fig:frankaW}.

{The proposed framework also demonstrated strong computational scalability across systems with substantially different dimensionality. Once the demonstration data is collected, the STTs for each state dimension were computed in parallel, enabling efficient offline tube synthesis. In our experiments, the offline STT generation required approximately $5.655$ s for the mobile robot ($x \in \mathbb{R}^2$) and $5.794$ s for the Franka manipulator ($x \in \mathbb{R}^7$). During online execution, the proposed closed-form controller required only $0.018$ ms per control step for the mobile robot and $0.021$ ms per control step for the manipulator. The comparable offline and online computation times across systems with substantially different dimensionality highlight the computational efficiency and scalability of the proposed framework.}

Across both platforms, these results demonstrate that the STT-LfD framework effectively translates the probabilistic specifications learned from demonstrations into robust and computationally efficient hardware execution without requiring explicit system identification or precise knowledge of the underlying model parameters.

\begin{table*}[t]
\centering
\caption{Quantitative comparison of STT-LfD and baseline controllers on the Franka Research~3 manipulator.}
\label{tab:baseline_comparison}
\setlength{\tabcolsep}{5pt}
\renewcommand{\arraystretch}{1.15}
\begin{tabular}{lccccc}
\toprule
\textbf{Method} &
\multicolumn{2}{c}{\textbf{Tracking Error (mm)}} &
\multicolumn{2}{c}{\textbf{Max Torque (N$\cdot$m)}} &
\textbf{Comp. Time (ms)} \\
\cmidrule(lr){2-3} \cmidrule(lr){4-5}
& \textbf{Nominal} & \textbf{Mass Variation} &
\textbf{Nominal} & \textbf{Mass Variation} & \\
\midrule
STT-LfD (ours) &
$\mathbf{2.83 \pm 0.09}$ &
$\mathbf{4.24 \pm 0.08}$ &
$\mathbf{1.93 \pm 0.001}$ &
$\mathbf{2.74 \pm 0.002}$ &
$\mathbf{0.021 \pm 0.007}$ \\

KMP+LQR \cite{silverio2019uncertainty} &
$10.08 \pm 0.02$ &
$53.85 \pm 0.47$ &
$8.43 \pm 2.36$ &
$9.25 \pm 1.97$ &
$15.941 \pm 2.83$ \\

DMP+MPC \cite{hu2019mobile} &
$10.64 \pm 0.05$ &
$55.58 \pm 1.01$ &
$7.81 \pm 2.74$ &
$7.51 \pm 3.13$ &
$22.733 \pm 1.31$ \\

{NODE-CLF-CBF} \cite{nawaz2024learning} &
$13.44 \pm 0.108$ &
$18.78 \pm 0.02$ &
$1.37 \pm 0.019$ &
$2.52\pm 0.12$ &
$1.172 \pm 0.35$ \\

{S2NNDS} \cite{binny2026safe} &
$17.19 \pm 0.011$ &
$27.39 \pm 0.09$ &
$1.853 \pm 1.076$ &
$3.206\pm 1.54$ &
$1.486 \pm 0.14$ \\
\bottomrule
\end{tabular}
\end{table*}

\subsection{Sensitivity Analysis of Tube Width Parameter:}
The tube width parameter $\lambda$ determines the width of the learned STTs, thereby controlling the trade-off between tracking precision and control effort. To analyze its effect, we varied $\lambda$ from $0.2$ to $2.4$ for the mobile robot figure-eight trajectory task. {For each value of $\lambda$, the experiment was repeated over $10$ trials under bounded random disturbances, and the reported results correspond to the mean $\pm$ one standard deviation.}

The quantitative results are shown in Fig.~\ref{fig:lambda_sensitivity}. As $\lambda$ increases, the RMS tracking error increases approximately linearly due to the wider admissible execution region, whereas the RMS control effort decreases approximately exponentially since the controller requires less aggressive corrective action near the tube boundaries. Based on this trade-off, $\lambda=1.2$ was empirically selected as a balanced operating point between tracking precision and control effort.


\subsection{Comparative Performance Analysis and Discussion}
To benchmark STT-LfD against both existing approaches, we conducted 10 trials for each of these methods: KMP+LQR \cite{silverio2019uncertainty}, DMP+MPC \cite{hu2019mobile}, NODE-CLF-CBF \cite{nawaz2024learning} and S2NNDS \cite{binny2026safe}. 
{For a fair comparison, all methods used the same expert demonstrations and were evaluated under identical task specifications, control frequency, input limits, and mass variation, while each baseline was tuned following its original implementation.}
The results are summarized in Table~\ref{tab:baseline_comparison}.

\subsubsection{Robustness to Model Perturbations} 
The most significant disparity between the methods is observed under dynamic uncertainty. When a 1 kg load was added to the manipulator, the tracking error of the model-based baselines (KMP+LQR and DMP+MPC) increased substantially, exceeding $50$~mm. 
{While NODE-CLF-CBF and S2NNDS showed better robustness, their tracking errors still increased to $18.78$~mm and $27.39$~mm, respectively.}
In contrast, STT-LfD maintained a tracking error of only $4.24\pm0.08$~mm under the same load variation.
This robust performance is a result of the controller's ability to compensate for unknown EL terms in real time.

\subsubsection{Impact of Spatiotemporal Precision}
During the precise drawing segments, the HGP-based synthesis generated a narrow STT due to low expert variability. STT-LfD successfully enforced these tighter constraints by increasing control effort proportionally to the normalized error. Conversely, the baselines do not explicitly encode demonstration variability as time-varying admissible execution regions.

\subsubsection{Computational Efficiency for Real-Time Execution}
The proposed controller required only $0.021$~ms per control step, compared to $15.94$~ms, $22.73$~ms, {$1.17$~ms, and $1.49$~ms} for KMP+LQR, DMP+MPC, {NODE-CLF-CBF, and S2NNDS}, respectively. This improvement stems from the closed-form nature of the controller, which avoids online optimization and iterative neural-network evaluations during execution.

\textbf{Overall Observation:}
By combining learning and control into a single framework, STT-LfD avoids the modeling mismatch inherent in decoupled designs. The learned STT acts as a time-varying safety specification, and the closed-form controller ensures that any trajectory executed by the robot is not only kinematically feasible but also remains within the expert's original intent envelope. 
All experimentation videos are available at \href{https://tinyurl.com/STT-LfD}{https://tinyurl.com/STT-LfD}.

\begin{remark}
    {For quantitative comparison, all methods were evaluated in the same Gazebo simulation environment under identical task and disturbance conditions. In addition, STT-LfD was validated on real hardware to demonstrate practical deployment. Hardware experiments were not conducted for the baseline methods, as they do not provide formal forward-invariance guarantees and exhibited significant performance degradation under model perturbations.}
\end{remark}

\section{Conclusion}\label{sec:conclusion}
This letter introduces STT-LfD, a unified framework that connects probabilistic learning with robust control. We transform demonstrations into STTs through HGP, and capture the expert’s intent. We then propose a input-constrained closed-form controller that does not require explicit system identification. Hardware results demonstrate that STT-LfD allows accurate, safe, and computationally efficient task reproduction, even for systems operating under dynamic perturbations and uncertain system parameters. 
{While the current framework focuses on learning and reproducing demonstrated behaviors in static environments, in future work, we plan to extend the framework to dynamically changing environments through online adaptation of the learned STTs for real-time obstacle avoidance and multi-robot coordination. We also plan to develop systematic methods for selecting the tube width parameter $\lambda$.}
{Another direction is to automatically adapt the temporal parameterization of the learned STTs, enabling demonstrations collected at arbitrary execution speeds to be transformed into dynamically feasible task specifications that satisfy the actuator constraints of the target platform.}

\bibliographystyle{ieeetr} 
\bibliography{sources} 

\appendix
\section{Bounded Transformation Functions}\label{sec:clamp}

The bounded transformation function $\Psi:\R^n \rightarrow \R^n$ is a smooth mapping that ensures the control inputs remain within specified limits while preserving desired behavior. We introduce and discuss the use cases of two categories of such functions: (i) saturation and (ii) zeroing.

\subsection{Saturation Transformation Function}
Saturation-type functions are defined as:
$\Psi(s) = [\Psi_1(s_1), \ldots, \Psi_n(s_n)]^\top$, where for all $i = [1;n]$:
$$\Psi_i(s_i) =  
    \begin{cases} 
      -1, & s_i \in (-\infty,-1], \\
      0, & s_i = 0, \\
      1, & s_i \in [1,\infty),
    \end{cases}
$$
and $\Psi_i(s_i)$ is non-decreasing for all $s_i \in (-\infty,\infty)$.

These functions saturate at $\pm1$ when their input exceeds $\pm1$. This saturation ensures robustness: even if the error grows beyond bounds due to disturbances, the system still receives control input that pushes it back toward the safe region.

\subsection{Zeroing Transformation Function}
Zeroing functions are also defined component-wise: 
$\Psi(s) = [\Psi_1(s_1), \ldots, \Psi_n(s_n)]^\top$, where for all $i = [1;n]$:
$$\Psi_i(s_i) = 
\begin{cases} 
  -1, & s_i = -1, \\
  0, & s_i = 0, \\
  1, & s_i = 1,
\end{cases}
$$
$$\lim_{s_i \rightarrow \pm\infty} \Psi_i(s_i) = 0 \text{ and $\Psi_i(s_i)$ is non-decreasing on $(-1,1)$}.$$

Unlike saturation functions, zeroing functions decay to zero outside $[-1,1]$. This behavior is useful in safety-critical settings where large errors may indicate danger. Reducing the control input to zero in such cases can prevent unsafe or unstable responses.

\section{Additional Manipulator Experiments}\label{sec:roboface}

\begin{figure}[t]
    \centering
    \includegraphics[width=0.9\textwidth]{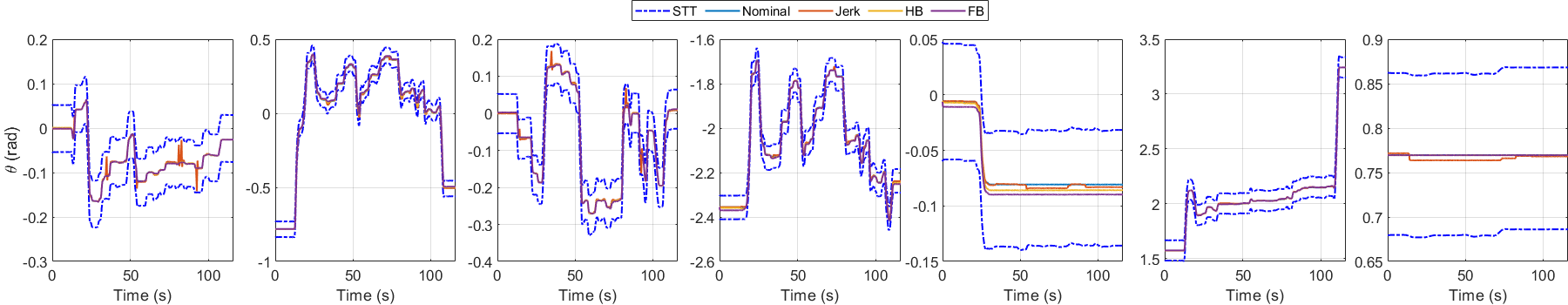}
    \caption{Demonstrated Data and Learned Spatiotoemporal Tubes for the 7-DOF Manipulator.}
    \label{fig:tubefranka}
\end{figure}

\begin{figure}[t]
    \centering
    \includegraphics[width=0.9\textwidth]{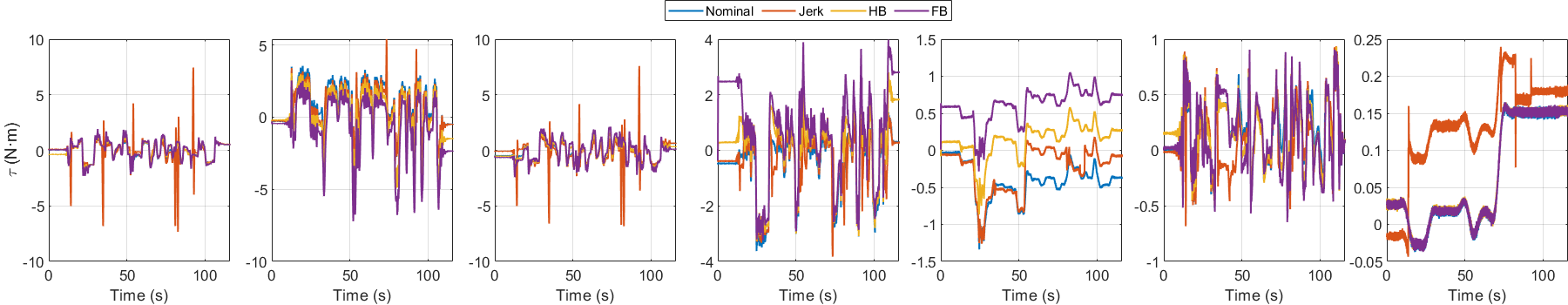}
    \caption{Torque Input to 7-DOF Manipulator.}
    \label{fig:confranka}
\end{figure}

\begin{figure}[t]
    \centering
    \includegraphics[width=0.7\textwidth]{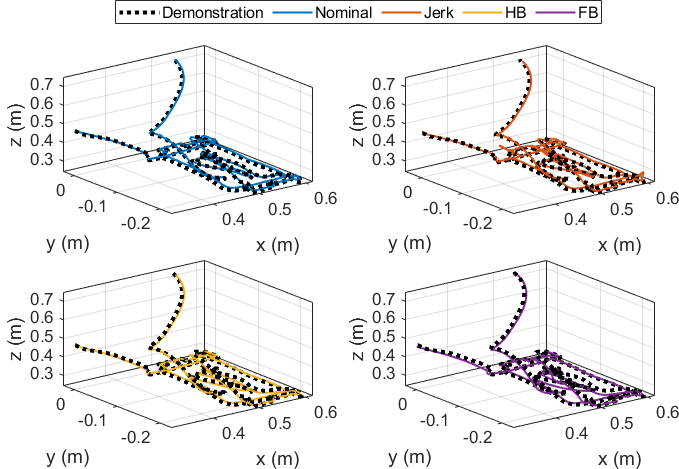}
    \caption{End Effector Trajectory of 7-DOF Manipulator.}
    \label{fig:eefranka}
\end{figure}

This case study involves a Franka Emika Research 3 robot, a lightweight 7-degree-of-freedom manipulator designed for safe human-robot collaboration \cite{FRANKA}. For this study, the robot’s maximum joint velocity and torque limits were set to $\overline{v} = 6$ rad/s and $\overline{\tau} = 8$ N-m, respectively \cite{FRANKA2, FRANKA3} for all the joints.

A kinesthetic demonstration was conducted starting from the robot’s default home configuration. During the demonstration, the operator physically guided the arm through the desired trajectory while joint positions were continuously recorded. STTs for each of the seven joints were then learned parallely with $\lambda = 1.2$. The resulting tubes are shown in Figure~\ref{fig:tubefranka}.

The approximation-free, feedback control strategy in Equation~\eqref{eqn:con} was subsequently deployed to replicate the demonstrated motion. To evaluate the system’s robustness, we conducted experiments under three conditions: (i) nominal operation without disturbances, (ii) altered dynamics by attaching either a half-filled or a fully filled water bottle to one of the links, and  (iii) external disturbances introduced by applying sudden jerks during task execution. Figure~\ref{fig:confranka} shows the resulting joint torques for each scenario. Hardware demonstration videos for these experiments are available at Link.

\end{document}